\documentclass[sigconf]{aamas}  

\AtBeginDocument{%
  \providecommand\BibTeX{{%
    \normalfont B\kern-0.5em{\scshape i\kern-0.25em b}\kern-0.8em\TeX}}}
    
\usepackage{booktabs}    
\usepackage{flushend} 
\usepackage{amsthm}
\usepackage{amsmath}
\usepackage{subfig}
\usepackage{amssymb}
\usepackage{graphicx}
\usepackage[ruled,vlined,linesnumbered,noend]{algorithm2e}

\setcopyright{ifaamas}  
\copyrightyear{2020} 
\acmYear{2020} 
\acmDOI{} 
\acmPrice{} 
\acmISBN{} 
\acmConference[AAMAS'20]{Proc.\@ of the 19th International Conference on Autonomous Agents and Multiagent Systems (AAMAS 2020)}{May 9--13, 2020}{Auckland, New Zealand}{B.~An, N.~Yorke-Smith, A.~El~Fallah~Seghrouchni, G.~Sukthankar (eds.)}  

\newtheorem{defn}{Definition}[]
\newtheorem{lem}{Lemma}[]
\newcommand{\foreps}{$episode=1$ \KwTo $M$}
\newcommand{\fort}{$t=1$ \KwTo $T$}
\newcommand{\forn}{$n=1$ \KwTo $N$}

\begin{document}

\title{Learning an Interpretable Traffic Signal Control Policy}  


\author{James Ault}
\affiliation{%
  \institution{Texas A\&M University}
  \city{College Station} 
  \state{Texas} 
  \postcode{77840}
}
\email{jault@tamu.edu}
\author{Josiah P. Hanna}
\affiliation{%
  \institution{University of Edinburgh}
  \city{Edinburgh} 
  \country{U.K.}
}
\email{josiah.hanna@ed.ac.uk}
\author{Guni Sharon}
\affiliation{%
  \institution{Texas A\&M University}
  \city{College Station} 
  \state{Texas} 
  \postcode{77840}
}
\email{guni@tamu.edu}

\begin{abstract}
  Signalized intersections are managed by controllers that assign right of way (green, yellow, and red lights) to non-conflicting directions. Optimizing the actuation policy of such controllers is expected to alleviate traffic congestion and its adverse impact. Given such a safety-critical domain, the affiliated actuation policy is required to be interpretable in a way that can be understood and regulated by a human. This paper presents and analyzes several on-line optimization techniques for tuning interpretable control functions. Although these techniques are defined in a general way, this paper assumes a specific class of interpretable control functions (polynomial functions) for analysis purposes. We show that such an interpretable policy function can be as effective as a deep neural network for approximating an optimized signal actuation policy. We present empirical evidence that supports the use of value-based reinforcement learning for on-line training of the control function. Specifically, we present and study three variants of the \textit{Deep Q-learning} algorithm that allow the training of an interpretable policy function. Our \textit{Deep Regulatable Hardmax Q-learning} variant is shown to be particularly effective in optimizing our interpretable actuation policy, resulting in up to 19.4\% reduced vehicles delay compared to commonly deployed actuated signal controllers.
\end{abstract}

%


\maketitle


\section{Introduction}

Travel time studies in urban areas show that 12--55\% of commute travel time is due to delays induced by signalized intersections (stopped or approach delay)~\cite{levinson1998speed,tirachini2013estimation}. Hence, optimized signal controllers have the potential of reducing commute time, traffic congestion, emissions, and fuel consumption, while requiring minimal infrastructure changes. 

Recent publications~\cite{laval2019large,li2016traffic,van2016coordinated} 
proposed to utilize state--of--the--art reinforcement learning algorithms for online optimization of signal controllers.
Such previous work showed a potential reduction of up to 73\% in vehicle delays when compared to fixed--time actuation~\cite{mousavi2017traffic}. 
Despite showing compelling empirical results, the controllers defined in such previous work have little applicability in the real world since their underlying control function is based on a deep neural network (DNN). While providing flexible and powerful function approximations, DNNs lack an interpretable inference process~\cite{samek2017explainable} which might prevent the implementation of related controllers in practice.
Given their liability for drivers' safety and mobility, governmental transportation agencies are conservative in requiring that signal controllers are interpretable and regulated. Consequently, this paper focuses on defining and studying self--optimizing, regulatable signal control policies.

The following contributions are made, which, to the best of our knowledge, were not addressed in previous work:

\noindent
\begin{enumerate}

\item Define and justify a regulatable control function for the signal control domain.
    
\item Define and study the effectiveness of a domain specific regulatable function when compared to a DNN--based policy.

\item Study the effectiveness of different optimization methods for online training of a signal control policy. Namely, \textit{Covariance Matrix Adaptation Evolution Strategy} (\textsc{cma--es})~\cite{hansen2001completely}, \textit{Proximal Policy Optimization} (PPO)~\cite{schulman2017proximal}, and \textit{Deep Q--learning} (DQN)~\cite{mnih2015human}.

\item Develop three variants of the DQN algorithm that utilize and train a regulatable control function. These variants are denoted \textit{Deep Regulatable Q--learning} (DRQ), \textit{deep regulatable softmax Q-learning} (DRSQ), and \textit{deep regulatable hardmax Q--learning} (DRHQ).

\item Test the performance of the aforementioned optimization methods through computer--based simulation of a real--life intersection and observed demand.

\item Compare the proposed methods with the commonly deployed \textit{actuated signal controller}~\cite{bonneson2011traffic}. As opposed to previous work that were compared to, the less effective, fixed--timing actuation policy. 
\end{enumerate}

We show that a designed regulatable control function can reduce traffic delays by up to 30\% when compared to common actuated signal controllers. Moreover, such a regulatable control function is shown to be competitive with the performance of a DNN based controller.
Next, we turn to study online optimization approaches for signal controllers. Our empirical study shows that a value--based approach (specifically DQN) converges faster, and to a better policy, compared to a policy--gradient approach (specifically PPO). Our regulatable Q--learning variant, DRHQ, is shown to result in a policy that is competitive with commonly deployed actuated signal controllers after a single training episode. The control policy is further improved in successive episodes, reaching up to 19.4\% reduced delays.

\section{Background and Related Work}

This section provides the necessary background which includes the domain description, Markov decision processes, reinforcement learning, and related work.


\subsection{The traffic signal domain}

A signalized intersection is composed of incoming and outgoing roads where each road is affiliated with one or more lanes. 
Each signalized intersection is assigned a set of phases, $\Phi$. Each phase, $\varphi \in \Phi$, is affiliated with a specific traffic movement through the intersection, as illustrated in Figure~\ref{fig:phases}.
Two phases are defined to be in \textit{conflict} if they cannot be enabled simultaneously (their affiliated traffic movement is intersecting). 
For example, in the phase allocation presented in Figure~\ref{fig:phases}, $\varphi_2 \text{ and } \varphi_1$ are conflicting phases.

\begin{center}
\begin{figure}[t]
    \includegraphics[trim=0 0 0 0, clip, width=0.85\columnwidth]{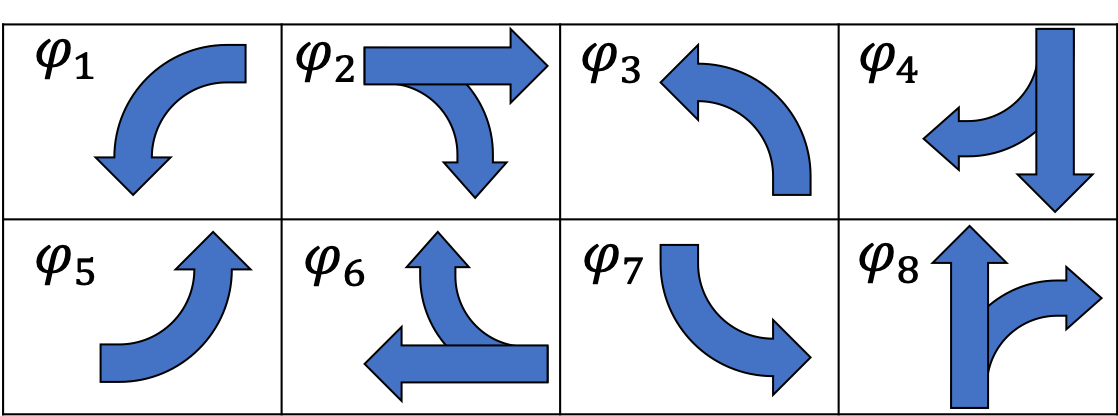}
    \caption{Common phase allocation for a 4--way intersection.}
    \label{fig:phases}%
\end{figure}
\end{center}
\subsection{Reinforcement learning}

In reinforcement learning (RL) an agent is assumed to learn through interactions with the environment. The environment is commonly modeled as a Markov decision process (MDP) which is defined by: $\mathcal{S}$ -- the state space, $\mathcal{A}$ - the action space, $\mathcal{P}(s_t,a,s_{t+1})$ -- the transition function of the form $\mathcal{S} \times \mathcal{A} \times \mathcal{S} \rightarrow Prob$, $R(s,a)$ -- the reward function of the form $\mathcal{S} \times \mathcal{A} \rightarrow \mathbb{R}$, and $\gamma$ -- the discount factor. The agent is assumed to follow an internal policy $\pi$ which maps states to actions, i.e., $\mathcal{S} \rightarrow \mathcal{A}$. The agent's chosen action ($a_t$) at the current state ($s_t$) affects the environment such that a new state emerges ($s_{t+1}$) as well as some reward ($r_t$) representing the immediate utility gained from performing action $a$ at state $s$, given by $R(s,a)$. The observed reward is used to tune the policy such that the expected sum of discounted reward, $J_\pi = \sum_t \gamma^t r_t$, is maximized. The policy $argmax_\pi[J_\pi]$ is the optimal policy denoted $\pi^*$. 

Recent publications~\cite{van2016coordinated,shabestary2018deep} suggest applying one RL approach, Deep Q--learning (DQN)~\cite{mnih2015human}, for training and operating signal controllers. In DQN, a deep neural network (DNN) is trained to map state--action pairs to a real number denoted the $Q$--value i.e., $\mathcal{S} \times \mathcal{A} \rightarrow \mathbb{R}$. The $Q$--value for a given state--action pair, $(s,a)$, represents the sum of discounted reward after performing $a$ at $s$ and then following the optimal policy ($\pi^*$). Algorithm 1 presents the pseudocode for DQN (for now, ignore lines in red i.e., Lines~\ref{ln:init_G},~\ref{ln:set_aG},~\ref{ln:begin_forn}--\ref{ln:end_forn}).   
For each training episode, DQN sets the initial state in Line \ref{ln:init_seq}. Next, for each time step within an episode, an action is chosen in an epsilon greedy manner (Lines 8--9) where the greedy action is the one with the maximal $Q$--value as approximated by the Q--function. The chosen action is executed and the outcome (immediate reward and next state) is stored in the replay memory (Lines 11--12). Next, a minibatch of transitions (state, action, immediate reward, next state) is sampled and the Q--function weights ($\theta$) are updated such that the squared error from the target $Q$--value, $y_j$, is reduced (Lines 13--15). The target $Q$--value ($y_j$) is computed using the temporal difference propagated from the next state~\cite{singh1996reinforcement}. Finally (Line \ref{ln:tar_upd}), every $C$ steps, the target Q--function ($\hat{Q}$) is set to equal the trained Q--function ($Q$). Doing such a delayed update is justified by Mnih et al.~ \cite{mnih2015human} as a way to reduce oscillations or divergence of the policy.

\subsection{Related work}

Previous works have examined reinforcement learning algorithms for online optimization of signal controllers. Unfortunately, the applicability of these protocols is questionable due to various drawbacks: \textbf{(a)} long and unsafe tuning process~\cite{van2016coordinated,shabestary2018deep,liang2018deep,mousavi2017traffic,genders2016using}
, \textbf{(b)} cumbersome policies that cannot be easily interpreted and regulated (mostly relying on deep artificial neural networks) thereby providing limited liability~\cite{van2016coordinated,shabestary2018deep,liang2018deep,mousavi2017traffic,genders2016using,wei2018intellilight,li2016traffic}
, \textbf{(c)} limited scalability (with regards to the number of managed phases)~\cite{van2016coordinated,wei2018intellilight,mousavi2017traffic,li2016traffic}
, \textbf{(d)} reliance on (currently) unrealistic traffic sensing capabilities~\cite{wei2018intellilight,mousavi2017traffic}, \textbf{(e)} experiments evaluated on unrealistic scenarios such as synthetic traffic demand~\cite{mousavi2017traffic,bakker2010traffic,genders2016using,li2016traffic,liang2018deep,abdoos2011traffic,prabuchandran2014multi}
, and handling only through traffic (no turning vehicles)~\cite{mousavi2017traffic,li2016traffic}.

A line of previous work
~\cite{wang2019large,laval2019large,prabuchandran2014multi, aslani2018traffic,van2016coordinated} focused on controlling and coordinating a set of signal controllers. The goal of such work is optimizing traffic flow over a road network that consists of several signalized intersections. Such a multiagent control manifests a combinatorial action space which results in limited scalability as well as slow and inefficient learning.

While some of the aforementioned publications presented compelling results, none yielded a control function that is human interpretable and liable. Given liability and regularity constraints faced by governmental transportation agencies, signal control functions will most likely not be adopted unless they can be interpreted and regulated.

Interpretability of deep reinforcement learning has recently been examined. Similar to this work, techniques for mimicking deep neural networks with  interpretable actuation were suggested. One work~\cite{liu2018toward} suggests using a variant of \textit{continuous U trees}, however this approach requires that the affiliated deep Q--network be trained to convergence initially. As a result, the control policy remains uninterpretable for the duration of the training episodes. Consequently, such an approach cannot be utilized in safety critical domains that require online learning.
This shortcoming is shared by other works \cite{hein2018interpretable, verma2018programmatically}, each requiring learning a model that is based on a deep neural net prior to utilizing an interpretable controller.

\section{Problem Definition}
\label{sec:prob-definition}

The focus of this paper is around online optimization of a human interpretable signal controller. Following previous work ~\cite{shabestary2018deep, wei2018intellilight, prashanth2011reinforcement} we model this problem as an MDP. The state space is defined by the possible input assignments. The set of available actions per state is defined by the possible output assignments and constraints. Constraints are considered as part of the transition function of the environment. For example, actions do not include yellow lights, these are set automatically when required for safety. The transition function is provided by the environment. Reward is defined by reduction in accumulated traffic delay. The controller's operation is defined as follows:

\noindent 
\textbf{Input:} \textbf{(1)} current signal assignment (green, yellow, and red assignment for each phase), for each lane: \textbf{(2)} number of approaching vehicles, \textbf{(3)} stopped vehicles accumulated waiting time, \textbf{(4)} number of stopped vehicles, and \textbf{(5)} average speed of approaching vehicles. 
Note that all inputs are necessarily non--negative.

\noindent 
\textbf{Output:} next signal assignment for a duration of time equal to a given minimum phase length. Signal assignments are abstracted to phases which group individual assignments into traffic movements. Non-conflicting pairs of phases give a complete signal assignment.

\noindent 
\textbf{Constraints:} \textbf{(1)} right--of--passage cannot be assigned to conflicting phases, \textbf{(2)} a yellow signal must appear for a predefined time interval between red and green signals. 

\noindent 
\textbf{Assumptions: }
\textbf{(1)} the intersection's layout is known, i.e., incoming/outgoing lanes, phases, and conflicting phases, \textbf{(2)} real--time sensing of incoming traffic as specified in the problem input.
These sensing assumptions are reasonable given latest advances in traffic sensing technology, namely, radar~\cite{samczynski2011concept}, Wi--Fi scanning~\cite{kostakos2013traffic}, drivers' smartphones~\cite{mohan2008nericell}, connected vehicles~\cite{wan2016mobile}, image processing~\cite{bautista2016convolutional}, infrared sensing~\cite{iwasaki2011robust}, and Synthetic Aperture Radar (SAR) satellites~\cite{meyer2006performance}.

\noindent 
\textbf{Desiderata: }
The signal controller is expected to assign right--of--passage such that: \textbf{(1)} the average delay suffered by incoming vehicles is minimized \textbf{(2)} the actuation policy can be interpreted and regulated by a human (a precise definition is given next).

\subsection*{Regulatable signal controller policy}

Control of safety critical tasks in general, and signal control specifically, require interpretable policies for liability and performance guarantee purposes. Unfortunately, there is little agreement on the meaning of interpretability. Ahmad et al~\citeyear{ahmad2018interpretable} state that ``The choice of interpretable models depends upon the application and use case for which explanations are required". Consequently, this section provides a definition for an interpretable model for the signal domain. 

A signal control policy is defined as, $\pi(s;\theta) \rightarrow (\Phi^g, \Phi^y, \Phi^r )$, mapping a given traffic state, $s$, and a set of parameters, $\theta$, to a signal assignment, that is three sets of phases representing green, yellow, and red signal assignments, $\Phi^g,~ \Phi^y,~ \Phi^r$, respectively. A valid signal assignment is one that results in no conflicting traffic movements.  The traffic state, $s$, is defined by the provided sensors input. The control function's parameters, $\theta$, should be tuned such that the control policy yields optimized performance. 

We define a precedence function, $g(s,\Phi;\theta') \rightarrow \mathbb{R}$, mapping a given traffic state, $s$, a set of parameters, $\theta'$, and a non--conflicting set of signal phases, $\Phi$, to a real number representing the precedence of assigning right--of--passage to $\Phi$. Note that a $Q$-function~\cite{mnih2015human} could serve as a precedence function but a precedence function is not necessarily a $Q$-function.
Such a precedence function suggests a control policy where the chosen signal assignment in state $s$, is ${\text{argmax}_{\Phi}}g(s,\Phi;\theta')$. A series of precedence functions, $G$, one for each action, defines a full order over all phase assignments. $\theta'$ is tuned such that the action with the highest precedence is the optimal action in expectation. 

\begin{defn}[Regulatable precedence function] \label{def:regulatable}
A precedence function $g$ is defined as regulatable, if for all state variables $s[i] \in s$, $\frac{\partial g}{\partial s[i]}$ exists and is either non--negative or non--positive for any possible assignment of $s$, i.e., the precedence function is monotonic in the state variables.
\end{defn} 

A control policy that is based on a regulatable precedence function is defined as a \textit{regulatable control policy}.
For such a policy, changes to the signal assignment can be intuitively interpreted as following changes in state variables, e.g., the right--of--passage was revoked from $\varphi_2$ and granted to $\varphi_4$ (as defined in Figure~\ref{fig:phases}) because the number of stopped Southbound vehicles increased while the number of such Eastbound vehicles decreased. Moreover, should policy adjustment be required, adding a weighting parameter to each state variable allows for intuitive tuning of the control function with regards to the specific state variable. 

Consider the example in Figure~\ref{fig:interp}, traffic has accumulated on the Eastbound left--turn lane, however the light configuration has not yet been switched. A simple regulatable control function can be optimized which chooses between two configurations, allowing Westbound traffic, or allowing Eastbound traffic. In this example, the precedence function is a simple summation of 4 input variables (Queue length, number of approaching vehicles, accumulated waiting time for stopped vehicles, and average speed). Each of these variables is affiliated with one tunable parameter. E.g., ``W-Through Queued" specifies the weight factor affiliated with the queue length variable for the Westbound trough lanes. We later generalize and discuss this type of controller in greater detail.
The table at the bottom of the figure specifies optimized values for the different tunable parameters. 
By inspecting the values assigned we can observe that the speed of approaching vehicles on the Westbound left--turn lane is the primary factor in the decision to maintain the current light configuration. However, it might be the case that the Eastbound left--turn lanes are too short to accommodate the typical number of vehicles. The parameters can be easily tuned to give higher precedence to clearing traffic from these lanes by increasing the weight parameter E-Left Queued.

\begin{figure}[h]
      \subfloat{  \includegraphics[width=0.9\columnwidth]{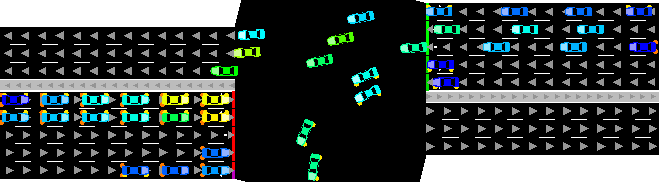}}
    \\ \
    	\subfloat{\small
    	\centering
    	\begin{tabular}{|l|r|r|r|r|r|}
    	    \hline
    		Action & Queued & Approaching & Waiting (s) & Speed & Value\\
    		\hline
    		W--Through & 0.0 & 2.1 & 0.0 & 0.33 &\\
    		W--Left & 3.64 & 0.77 & 4.65 & 11.04 & 22.53\\
    		\hline
    		E--Through & 1.25 & 0.0 & 0.73 & 0.0 &\\
    		E--Left & 5.32 & 1.67 & 10.34 & 0.0 & 19.31\\
    		\hline
        \end{tabular}
    	\normalsize}
	\caption{Optimized actuation policy for the presented intersection (top). The policy is based on a linear control function that assigns values to input variables presented in the table (bottom).}
	
	\label{fig:interp}
\end{figure}

This work makes the assumption that a regulatable control function is interpretable (a similar assumption was made in~\cite{hein2018interpretable}).

\section{Designing a Regulatable Precedence Function}



Based on the problem input (as defined in the problem definition) we define the following, phase dependent state variables, $s_\varphi[1,...,6]$: (1) number of stopped vehicles, (2) number of approaching vehicles, (3) cumulative stopped time, (4) average stopped time, (5) average queue length = stopped vehicles divided by the number of lanes, and (6) average speed for approaching vehicles. These phase dependent variables relate only to vehicles that are present on lanes served by phase $\varphi$. For example, $s_{\varphi_1}$ returns the number of stopped vehicles on the Eastbound left--turning lanes.
On top of these state variables, the proposed precedence function penalizes phase assignments that enforce clearance intervals. For instance, activating (assigning right--of--passage) the Southbound through phase straight after an active Eastbound through phase commonly requires a clearance interval. The clearance interval activation and duration is intersection dependent. For general guidelines see~\cite{bonneson2011traffic}.
Transitioning from a currently active phase set to another phase set, $\Phi$, triggers one of the following cases. 
\begin{enumerate}
\item \textbf{Full clearance} -- an interval where no phase is active (all red).
\item \textbf{Partial clearance} -- an interval where part of the phases in $\Phi$ are inactive (prior to becoming active).
\item \textbf{Permissive clearance} -- a short interval (shorter than in the partial clearance case) where part of the phases in $\Phi$ are inactive. Commonly due to a currently active permissive--left phase (when applicable).
\item \textbf{No clearance} -- no clearance is required.
\end{enumerate}
Each of these cases is affiliated with a flag, denoted $f_{1,...,4}$, that is set to 1 for the active case and 0 for the others. In a state where two cases are simultaneously active, only the one with the lowest index will be set to 1.



Our proposed precedence function can now be defined as:
\begin{equation} \label{eq:prec}
    g(s,\Phi ; \theta') = \sum_{\varphi \in \Phi} \sum_{i=1}^{6} \left( w^{\Phi\varphi}_i s[i] \right) ^{p^{\Phi\varphi}_i} \cdot \sum_{j=1}^{4} \left( w'^{\Phi}_j f_j \right) ^{p'^{\Phi}_j} 
\end{equation}

Where $\theta' = \{w,p,w',p'\}$.
In total, the designed function is composed of 6 state variables per $\varphi$ in $\Phi$, each with two affiliated tunable parameters, a weight $(w^{\Phi\varphi}_{1...6})$ and an exponent $(p^{\Phi\varphi}_{1...6})$. Also, each of the clearance flags ($f_{1...4}$) is affiliated with two tunable parameters, a weight $(w'^{\Phi}_{1...4})$ and an exponent $(p'^{\Phi}_{1...4})$. Each phase combination ($\Phi$), phase ($\varphi$), and index ($i$ or $j$) defines a unique tunable parameter. 
All together, this function defines 12 parameters per phase within a phase combination ($w^{\Phi\varphi},~p^{\Phi\varphi}$) and 8 parameters per phase combination ($w'^{\Phi},~p'^{\Phi}$). 
In the phase diagram defined in Figure~\ref{fig:phases} there are 8 phases which result in 8 sets of non--conflicting pairs. Namely, $\{(\varphi_1,\varphi_2) \bigtimes (\varphi_5,\varphi_6) \cup (\varphi_3,\varphi_4) \bigtimes (\varphi_7,\varphi_8)\}$. Consequently, the appropriate regulatable function $G$ (over all 8 phase pair precedence functions) will be composed of $8 (2 \cdot 12 + 8) = 256$ tunable parameters.

\begin{lem}
  \label{lem:regulatable}
    The precedence function defined in Equation~\ref{eq:prec} is regulatable according to Definition~\ref{def:regulatable}.
\end{lem}
\begin{proof} \ \\
$$\frac{\partial g}{\partial s[i]} = w^{\Phi\varphi}_{i} p^{\Phi\varphi}_i s[i]^{p^{\Phi\varphi}_i - 1}  \cdot \sum_{j=1}^{4} \left(w'^{\Phi}_j f_j\right)^{p'^{\Phi}_j}$$
$s[i] \ge 0$ for any $i \in \{1,...,6\}$ (see `input' in problem definition). 
Consequently, $s[i]^{p^{\Phi\varphi}_i - 1} \ge 0$. All other components of the partial derivative i.e.,  $w^{\Phi\varphi}_{i} p^{\Phi\varphi}_i  \cdot \sum_{j=1}^{4} \left(w'^{\Phi}_j f_j\right)^{p'^{\Phi}_j}$, are not dependent on any of the state variables and can be viewed as a constant ($\vartheta$). As a result, $\frac{\partial g}{\partial s[i]}$ would be either non--negative ($\vartheta \ge 0$), non--positive ($\vartheta \le 0$), or both ($\vartheta = 0$),  for any $i,~s$ and a given $\Phi,~\theta'$ assignment.


\end{proof}

The precedence function defined in Equation~\ref{eq:prec} is hereafter denoted the \textit{designed precedence function}. The affiliated control policy which returns $\text{argmax}_\Phi G(s,\Phi; \theta')$ is hereafter denoted the \textit{designed control policy}. Next, we discuss general techniques for online tuning of $\theta'$ such that the performance of any regulatable control policy (and specifically the designed control policy) is optimized.

\begin{algorithm}[t!]
\label{alg:DRQ}
\SetAlgoLined
\DontPrintSemicolon
Initialize replay memory $D$ to capacity $N$\;
Initialize $Q$ function with random weights $\theta$\;
Initialize target $\hat{Q}$ function with weights $\bar{\theta}=\theta$\textcolor{red}{\;
Initialize regulatable function $G$ with weights ${\widetilde{\theta}} = [1,...,1]$} \label{ln:init_G}\;
    \For{\foreps}{
        Initialize $s_1$ with observed state \label{ln:init_seq} \;
    	\For{\fort}{
        	With probability $\epsilon$ select a random action $a_t$ otherwise select:\;
        	~~~~$a_t= \text{argmax}_{a}Q(s_t,a;\theta)$ \# for DQN\textcolor{red}{ \label{ln:set_aQ}\;
        	~~~~$a_t= \text{argmax}_{a}G(s_t,a;\widetilde{\theta})$ \# for regulatable} \label{ln:set_aG}\;
        	
        	Execute action $a_t$ and observe reward $r_t$ and state $s_{t+1}$\;
        	
        	Store transition ($s_t,a_t,r_t,s_{t+1}$) in $D$\;
        	
        	Sample random minibatch of transitions $(s_j,a_j,r_j,s_{j+1})$ from $D$\;
        	
        	Set $y_j=r_j+\gamma\text{max}_{a'}\hat{Q}(s_{j+1},a';\bar{\theta})$\;
        	
        	Perform a gradient descent step on $(y_j-Q(s_j,a_j;\theta))^2$ with respect to $\theta$\;
        	
        	Every C steps reset $\hat{Q}=Q$ \label{ln:tar_upd} \;
        	
        	\BlankLine\color{red}
        	    \For{\forn}{ \label{ln:begin_forn}
            	    Sample random minibatch of transitions $(s_j,a_j,r_j,s_{j+1})$ from $D$ \label{ln:sample} \;
            	    
                	\Case{\textbf{DRQ}}{
                    	Set $y_j=r_j+\gamma\text{max}_{a'}\hat{Q}(s_{j+1},a';\theta)$ \label{ln:drq_bellman} \;
                    	Perform a gradient descent step on $(y_j-G(s_j,a_j;\widetilde{\theta}))^2$ with respect to $\widetilde{\theta}$ \label{ln:drq_opt} \;
                	}
                	\Case{\textbf{DRSQ}}{ \label{ln:drsq}
                	    Set $X_j = \text{softmax}_a\left(Q(s_j,\cdot ~;\theta)\right)$\;
                    	Set $Z_j = \text{softmax}_a\left(G(s_j, \cdot ~;\widetilde{\theta})\right)$\;
                    	Perform a gradient descent step on $-\sum_{a \in A}X_j[a]log(Z_j[a])$  w.r.t. $\widetilde{\theta}$ \label{ln:drsq_opt} \;
                	}
                	\Case{\textbf{DRHQ}}{ \label{ln:drhq}
                	    Set $X_j=\begin{cases}
                	        1 & \text{argmax}_{a'}Q(s_j,a';\theta)\\
                	        0 &  a \in A\setminus \{a'\}
                	    \end{cases}$ \label{ln:drhq_cases} \;
                	    Set $Z_j =  \text{softmax}_a\left(G(s_j, \cdot ~;\widetilde{\theta})\right)$\;
                    	Perform a gradient descent step on $-\sum_{a \in A}X_j[a]log(Z_j[a])$  w.r.t. $\widetilde{\theta}$ \label{ln:drhq_opt} \;
                	}
            	}\label{ln:end_forn}
    	}
    }
	\caption{DQN and \textcolor{red}{3 DRQ variants shown in red}} 
\end{algorithm}

\section{Parameter Tuning}



A line of publications~\cite{wei2018intellilight,liang2018deep,van2016coordinated} reported that the DQN algorithm is particularly suitable for online signal control optimization. Unfortunately, the underlying policy in DQN is not regulatable as it is based on a DNN (Line \ref{ln:set_aQ} in Algorithm 1). In order to bridge this gap, we suggest training a given regulatable policy function of the type $\text{argmax}_a G(s,a;\widetilde{\theta})$, to imitate the Q--network actuation i.e., $\text{argmax}_{a}Q(s,a;\theta)$. A simple approach would be to directly use Q--learning with a function approximator that is defined as the regulatable function. However, doing so was found to make reaching a reasonable policy infeasible under even artificially low demands. Such a simple function approximator is incapable of representing the required intermediate functions for learning action values over an extended period of time.

Consequently, other approaches for leveraging DQN to optimize regulatable policies are considered.
These approaches follow the pseudocode described in Algorithm 1. The lines in red (Lines~\ref{ln:init_G},~\ref{ln:set_aG},~\ref{ln:begin_forn}--\ref{ln:end_forn}) show the required additions on top of the original DQN algorithm. It is important to note that instead of selecting an action according to $\text{argmax}_{a}Q(s,a;\theta)$ (Line \ref{ln:set_aQ}), the regulatable version selects an action according to a regulatable policy, $\text{argmax}_a G(s,a;\widetilde{\theta})$ where $G$ is a previously initialized regulatable function (Line \ref{ln:init_G}). Replacing the actuator in the DQN algorithm (Line \ref{ln:set_aG} in lieu of Line \ref{ln:set_aQ}) is reasonable as DQN is an \textit{off--policy} algorithm, i.e., training the Q--function does not require that the same function interacts with the environment. Moreover, the powerful (DNN based) Q--function approximation that is trained by DQN can be used to train the regulatable function, $G$. Consequently, $G$ is repeatedly trained to minimize the error between $\text{argmax}_a G$ and $\text{argmax}_a Q$. This training is performed at every time step over $N$ random minibatches from the replay memory (Lines \ref{ln:begin_forn},\ref{ln:sample}). Next, we discuss 3 different strategies for training $G$.

\subsection{Deep Regulatable Q--Learning (DRQ)}

DRQ is our basic Regulatable Q--Learning variant where the parameters of the regulatable function $\widetilde{\theta}$ are tuned towards equivalency between $G$ and the $Q$--function (using SGD with a squared loss function over the provided minibatch, Line \ref{ln:drq_opt}). 
This variant follows the fact that if $G(s,a) = Q(s,a)$ for all $(s,a)$ then the required policy equivalency is achieved i.e., $\forall s, \text{argmax}_{a}G(s,a) = \text{argmax}_a Q(s,a)$.

Attempting to tune $\widetilde{\theta}$ such that the regulatable function $G$ would match the DNN--based $Q$--function may not be feasible in many cases. A regulatable function is more constrained and possesses far fewer tunable parameters compared to a DNN. As a result, a DNN is usually able to approximate a much larger set of functions compared to a regulatable approximator. Indeed, our empirical study found that the designed precedence function is very limited in its ability to approximate $Q$--values. 
However, setting the regulatable function, $G$, to imitate the action selection of DQN does not require that $G(s,a) \equiv Q(s,a)$. This understanding leads us to DRQ variants that provide extra flexibility with regards to the function approximated by $G$.

\subsection{Deep Regulatable Softmax Q--Learning (DRSQ)}

In DRSQ the parameters of the regulatable function $\widetilde{\theta}$ are tuned towards \textbf{proportional} equivalency between $G$ and the $Q$--function. This variant follows the fact that if $G(s,\cdot) \propto Q(s,\cdot)$ for all $s$ then the required policy equivalence is achieved i.e., $\forall s, \text{argmax}_{a}G(s,a) = \text{argmax}_a Q(s,a)$.
The proportionality values over all actions are standardized using the softmax function. As a result the SGD applied for tuning $\widetilde{\theta}$ takes a gradient based on the cross--entropy objective i.e., using the log loss function (Lines \ref{ln:drsq}--\ref{ln:drsq_opt}).

\subsection{Deep Regulatable Hardmax Q--Learning (DRHQ)}

DRHQ stems from the understanding that policy equivalency does not require full equivalency or even proportional equivalency between $G$ and the $Q$--function. In fact, for achieving policy equivalency it is sufficient to tune $\widetilde{\theta}$ directly towards $\text{argmax}_{a}$ equivalency between $G$ and the $Q$--function. This can be achieved by setting the target value for a given $s, a$ pair as 1 for $\text{argmax}_a Q(s,a)$ or 0 otherwise (Line \ref{ln:drhq_cases}). Next, SGD is used to tune $\widetilde{\theta}$ according to the log loss between the target values and $\text{softmax}\left(G(s,\cdot)\right)$. 

\subsection{Other tuning approaches}

The covariance matrix adaptation evolutionary strategy (\textsc{cma--es}) algorithm~\cite{hansen2001completely} is known to be a particularly effective parameter tuning approach. \textsc{cma--es} is specifically suitable for domains where the tunable parameters have a continuous value range. Moreover, \textsc{cma--es} is known for having few hyper--parameters with fairly low sensitivity. As a result, it is particularly appealing for testing and validating our designed policy. On the other hand, \textsc{cma--es} may be unsuitable for online tuning due to inefficient data sampling, requiring several full episodes for a single policy update step. Moreover, the erratic exploration of \textsc{cma--es}, though helpful in avoiding local optimums, is less suitable for safety critical domains. 

Another natural candidate for online parameter tuning is the policy gradient approach~\cite{sutton2018reinforcement}. Specifically, the Proximal--Policy Optimization (PPO) algorithm~\cite{schulman2017proximal} is particularly suitable for safety--critical domains as it encourages bounded policy gradient steps. PPO achieves this behaviour by clipping the expected advantage for large policy divergence. Consequently, PPO is expected to result in smooth and steady convergence but, on the other hand, is more prone to settle in a local optimum and achieve sub--optimal performance.

\section{Empirical Study}
The purpose of the empirical study is to evaluate and analyze the performance of the proposed designed control policy along with the affiliated online tuning algorithms. Specifically, the empirical study aims at answering the following questions:\\
\begin{enumerate}
    \item Can the designed control policy approximate a deep Q--learning optimized policy?
    
    \item Is a policy gradient optimization approach suitable for training the designed control policy?
    
    \item How do the regulatable Q--learning variants compare to the state--of--the--art, DQN--based signal controller?
\end{enumerate}



\subsection{Experimental settings}

\begin{figure*}%
    \centering
    \subfloat{{\includegraphics[height=4.3cm]{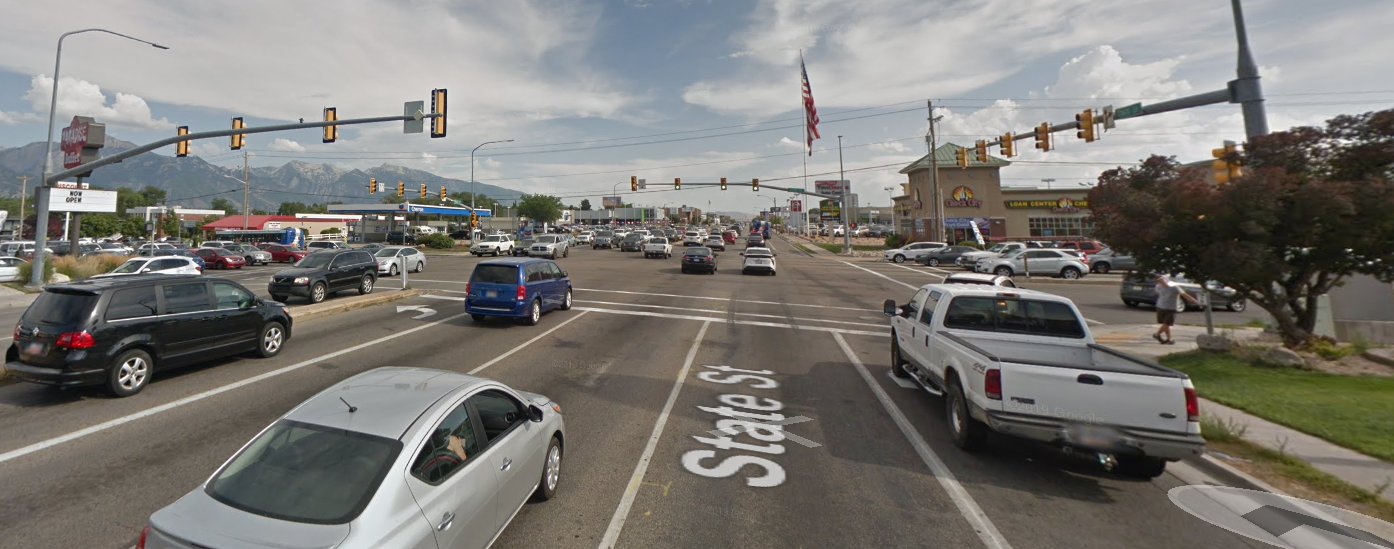} }}%
    \subfloat{{\includegraphics[height=4.3cm]{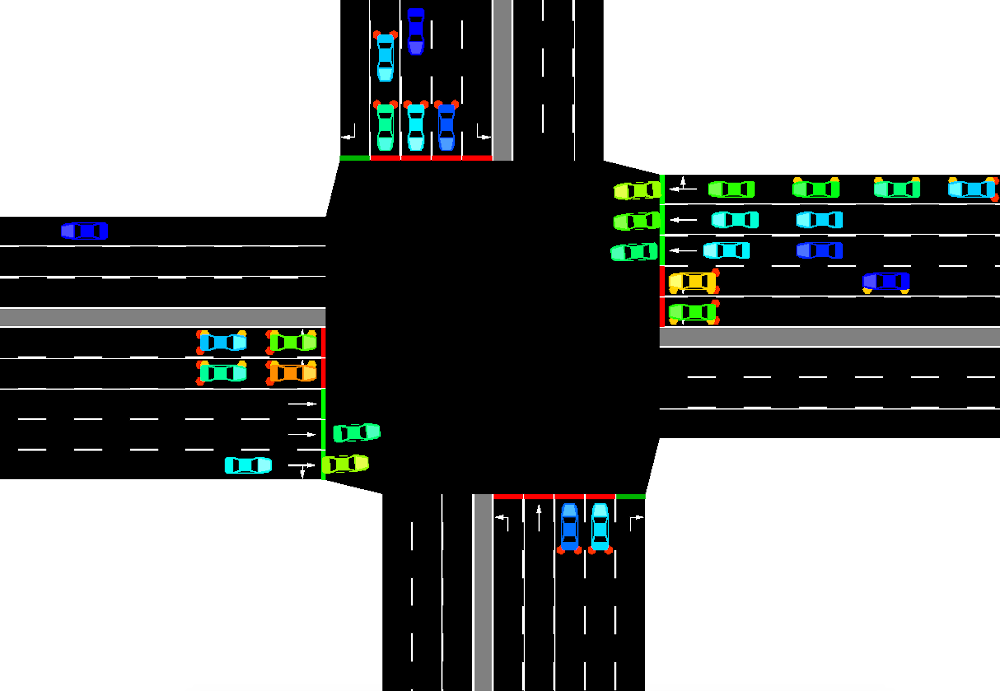} }}%
    \caption{A signalized intersection at State St \& E 4500 S, Murray, Utah (left), (picture credit: Google Maps), and an equivalent intersection modeled in SUMO (right).}%
    \label{fig:UDOT}%
\end{figure*}

The reported experiments rely on a well--established traffic simulator, Simulation of Urban MObility (SUMO)~\cite{behrisch2011sumo}, along with traffic scenarios that are based on real--life observations. The Utah department of transportation (UDOT) provides an open access database (\url{https://udottraffic.utah.gov/ATSPM}) specifying the observed traffic demand for 2092 signalized intersections. The database specifies the number of vehicles affiliated with each incoming, outgoing road combination in 5 minutes aggregation. The demand reported by UDOT is parsed into SUMO where vehicles are spawned with equal probability along equivalent 5 minute intervals. 

\begin{table}
    \label{tab:demands}
	\small
	\centering
	\begin{tabular}{|l|r|r|r|r|}
	\hline
		Demand & \multicolumn{1}{|c|}{Total} & \multicolumn{1}{|c|}{Avg Rate} & Low Rate & High Rate\\
		\hline
		Low & 45,112 & 1.04 & 0.59 & 1.29\\
		Medium & 51,298 & 1.19 & 0.76 & 1.42\\
		High & 61,261 & 1.42 & 0.98 & 1.59 \\
		\hline
    \end{tabular}
	\normalsize
	
	\caption{Traffic demand for three different days representing low, medium, and high traffic volumes. Each day is affiliated with: the total number of approaching vehicles (Total), approaching vehicle per second on average (Avg Rate), during off--peak (Low Rate), and during peak hour (High Rate).}
\end{table}

The reported results relate to a representative major intersection, State St \& E 4500 S, Murray, Utah. Source code for all experiments is available at: \textit{https://github.com/jault/StateStreetSumo}. This intersection is chosen as it is affiliated with high volumes of traffic arriving from two arterial roads. It typically receives more than 50,000 vehicles a day, peaking at 95 cars per minute in rush hour. Figure~\ref{fig:UDOT} provides a snapshot of the simulated intersection (right) and a picture from the actual intersection (left). This intersection is affiliated with 10 phases, the 8 reported in Figure~\ref{fig:phases}, and 2 North and South bound permissive left turns. The minimum phase length is set to 3 seconds for this intersection. These 10 phases form 11 unique pairs of non--conflicting phase combinations. As a result, the affiliated designed control policy has 352 tunable parameters. The UDOT database (Signal \#7157) specifies the affiliated clearance interval's activation and duration.

In order to examine various traffic conditions, demand from 3 different days is chosen for simulation, Wednesday -- May 1, Monday -- May 6, and Friday -- June 21, all of 2019. For each day, a 14 hour time period is considered from 7 A.M. to 9 P.M. These dates were chosen as representative examples of low (May 6th), medium (May 1st), and high (June 21st) traffic volumes. Table 1 presents the total number of vehicles that crossed the intersection on each day as well as the average number of vehicles arriving per second, during off--peak (Low Rate), and peak--hour (High Rate).  

The reported \textsc{cma--es} implementation is based on pycma~\cite{hansen2019pycma}, and the initial variance and population size were chosen to be 0.2 and 12 respectively. The simulated intersection controller is defined as an environment within OpenAI's GYM~\cite{1606.01540}, which gives a standard interface for reinforcement learning. Hyper--parameters for the online algorithms under all demand profiles aside from discount factor are identical throughout the different variants. The discount factor and epsilon were chosen empirically. Other values resulted in similar trends yet yielded slightly worse outcomes. Under low and medium demand the discount factor is set as 0.8. The discount is raised to 0.9 in high demands; this change was found to be important as clearing traffic from a set of lanes requires a larger planning horizon.

For DQN and the regulatable variants, the minibatch size is set to 32 and replay capacity at 100,000 transitions. The epsilon greedy action--selection probability in DQN and its variants is reduced from 0.05 to 0 (full exploit) after 20 training episodes resulting in the observed drop in the graphs. The Q--network is a DNN composed of 3 hidden layers with 64 units each, where the first layer is a 2x2 kernel convolutional layer grouping the input for lanes that belong to the same road. The Huber loss function is used in line with the original DQN work. Leaky ReLU activation is used for all layers along with the Adam optimizer with a step--size of 0.001, and decay rates $\beta_1$ and $\beta_2$ are set as 0.9 and 0.999 respectively. The same parameters were used for both optimizers in the two--stage regulatable variants. The PPO implementation follows the advantage actor--critic paradigm defining the actor by the regulatable function and the critic by a similar neural network as described prior with an alternative objective of estimating state advantages. The step size for each optimizer in this case are 0.001 for the actor and 0.001 for the critic with all other hyper parameters remaining the same. 


Finally, as a baseline for comparison, the results include performance measures for the commonly deployed actuated signal control~\cite{bonneson2011traffic}. This type of controller is provided by the SUMO simulator. Phases for the actuated controller are set in a fixed order, protected lefts followed by through traffic. The maximum green time is set to 300 seconds.

\begin{figure*}%
    \centering
    \subfloat{\includegraphics[width=\textwidth]{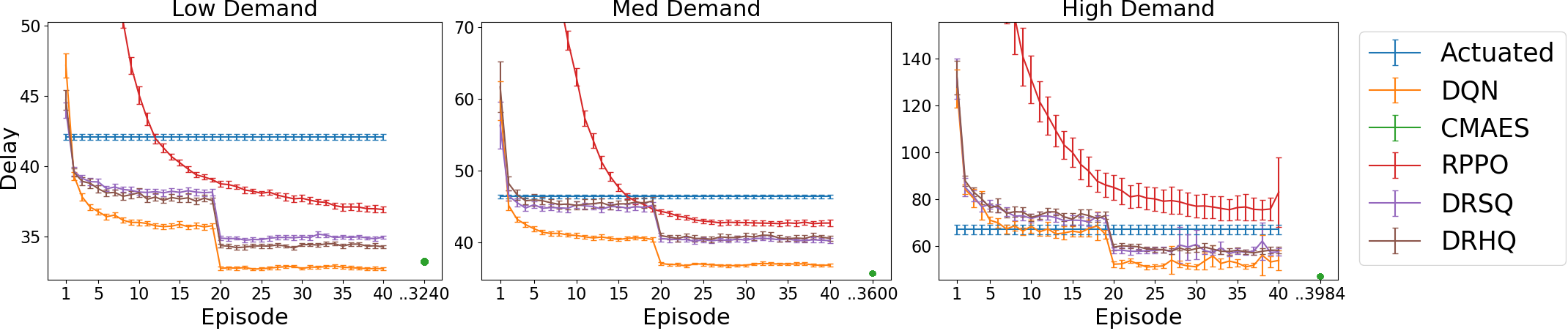}}
    \caption{Average seconds of delay caused by the intersection under each demand profile. In each experiment the 95\% confidence intervals per episode over a population of 30 trials with random seeds is given.}%
    \label{fig:delays}%
\end{figure*}

\subsection{Regulatable control function}

The first set of experiments aim at addressing the question: can the designed control policy approximate a deep Q--learning optimized policy? 

Given the stochastic nature of the signal domain and the combinatorial state space, it is  infeasible to compute the true optimal performance. Instead, we say that the control function can approximate an optimized policy if it results in performance that is competitive with a state--of--the--art, DNN based controller ~\cite{shabestary2018deep}. In order to address this question, a parameter tuning process is applied to the regulatable function. \textsc{cma--es} is chosen as the optimization approach since it is comparatively insensitive to hyper--parameter settings with just the population size and initial variance as hyper--parameters. Further, policy performance can be determined over entire episodes as \textsc{cma--es} is highly parallelizable. Finally, the simulator provides our goal of delay directly when vehicles exit the simulation. This allows for optimization of the designed control policy with respect to the delay rather than approximated through the observed waiting time as real--time algorithms require.

Figure~\ref{fig:delays} presents the average delay suffered by approaching vehicles for each of the three representative days as a function of the training episode (full day of traffic). For now, consider the DQN curve (representing the state--of--the--art, DNN based controller) and the \textsc{cma--es} datapoint (in green). As can be seen, \textsc{cma--es} with the designed regulatable control function achieves competitive delay measurements on all three traffic scenarios. Note that, \textsc{cma--es} is presented as a single data--point (post convergence), as it requires two orders of magnitude more samples in order to converge compared to the other approaches. The full tuning process cannot be fitted on the presented plots.

It is important to note that other, simpler regulatable control functions were also examined, specifically based on Polynomial~\cite{lagoudakis2003least}, and Fourier basis~\cite{konidaris2011value} function approximators. Both yielded extremely poor performance. After over 100 epochs either policy failed to complete the scenario by clearing all vehicles by the end of the simulated time period. The polynomial function showed some improvement, while Fourier basis was stagnant.

Despite its impressive ability to optimize the designed, regulatable function, \textsc{cma--es} is not practical for online optimization for two reasons. (1) Inefficient sampling -- each learning step (update of parameters' mean value) requires 24 episodes (24 full days of traffic). (2) Erratic exploration -- extreme parameters values are sampled leading to unacceptable performance during the tuning process. 
These drawbacks are immediately evident in Figure \ref{fig:cmaes} where the average delay over each epoch is presented. \textsc{cma--es} fails to meet actuated performance for at least 25 epochs (600 episodes) for the medium demand scenario (similar trends were observed for the high and low demand scenarios). The erratic exploration is a hindrance as some solutions entirely fail to complete the scenario or perform much worse than average. Performance doesn't become stable for nearly another 4,000 episodes (11 years of traffic). Nonetheless, these results are still valuable as they provide a promising lower bound estimate on the performance of the regulatable policy (as seen in Figure~\ref{fig:delays}).

\begin{figure}[h]
    \includegraphics[width=0.9\columnwidth]{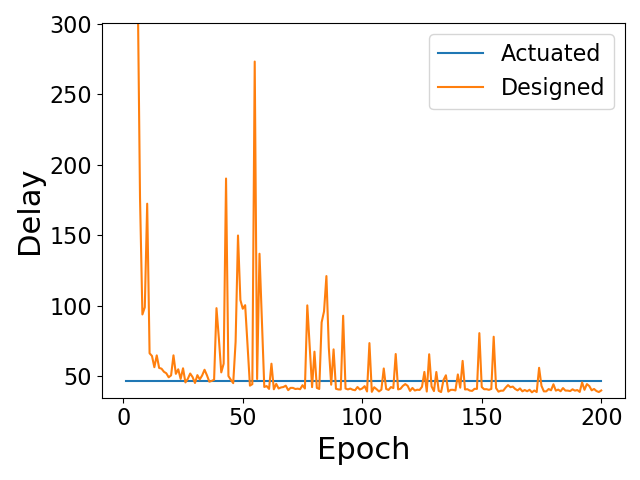}
    \caption{Average seconds of delay caused by the intersection under medium delay optimized by \textsc{cma--es} plotted as the mean of 24 episode epochs.}%
    \label{fig:cmaes}%
\end{figure}

\subsection{Policy gradient approach}

Policy--gradient approaches~\cite{williams1992simple} are known to be especially suitable for optimizing a given policy function while providing some bounds on the exploration rate~\cite{schulman2015trust,schulman2017proximal}. As a result, such algorithms are promising candidates for online tuning of regulatable functions in safety critical domains. 

Figure~\ref{fig:delays} presents the performance curve for tuning the designed control policy using the PPO algorithm. As expected, the learning curve is smooth and monotonic suggesting a safe exploration process. On the other hand, PPO requires about 15 episodes to reach actuated control performance in low and medium demand, while not even reaching actuated level performance in the case of high demand. As a result, such an approach is unlikely to be adopted in practice. Moreover, bounded exploration can cause PPO to converge to suboptimal local optimums as evidenced by the presented results.

\subsection{Q--learning}

Given that the DQN algorithm achieves state--of--the--art results in the signal control domain (yellow curve in Figure~\ref{fig:delays}), the next set of experiments is set to examine how the regulatable Q--learning variants compare. 

Figure~\ref{fig:delays} presents the performance curve for DRSQ (in purple) and DRHQ (in brown). As stated above, vanilla DRQ performs significantly worse and is, thus, not presented in this plot. Both DRSQ and DRHQ outperform PPO and, for low and medium demand, outperform actuated control after a single episode. On the other hand, it takes them 20 episodes to outperform actuated control for high demand (once epsilon is reduced to 0). In order to mitigate the long training time on high demand, Laval et al.~\citeyear{laval2019large} suggested to train the controller on low demand prior to applying it to high demand scenarios. Such an approach is expected to be specifically useful for training regulatable control functions.

DRHQ has a small, yet significant, advantage over DRSQ for low demand while yielding similar performance for medium and high demand. Finally, while DQN outperforms both, it is important to keep in mind that the policy induced by DQN cannot be interpreted and regulated. Consequently, DQN is unlikely to be implemented and serves only as an expected lower bound on regulatable performance.

\section{Summary and Conclusions}

This paper discusses online optimization of interpretable signal control policies. Unlike previous work that based such controllers on deep neural networks, this paper suggests utilizing a policy function that can be interpreted and regulated. A regulatable policy function is defined as one with a monotonic relation between each state variable and the precedence of a given action.
The following conclusions are drawn.

\noindent
\textbf{1.} A regulatable function can approximate an optimized policy in a way that is competitive with a deep neural network.

\noindent
\textbf{2.} A policy gradient approach is not suitable for training a regulatable function in this domain due to slow and sub--optimal convergence.

\noindent
\textbf{3.} A Q--learning approach which trains a regulatable function results in good performance both with regards to convergence speed and the final policy. The regulatable function should be trained to fit the hardmax action as provided by the underlying deep Q--network according to DRHQ.

Future work will examine techniques for warm starting the learning process by observing the operation of a currently deployed controller.


\bibliographystyle{ACM-Reference-Format}  
\bibliography{bibliography}  


\begin{thebibliography}{00}


\ifx \showCODEN    \undefined \def \showCODEN     #1{\unskip}     \fi
\ifx \showDOI      \undefined \def \showDOI       #1{#1}\fi
\ifx \showISBNx    \undefined \def \showISBNx     #1{\unskip}     \fi
\ifx \showISBNxiii \undefined \def \showISBNxiii  #1{\unskip}     \fi
\ifx \showISSN     \undefined \def \showISSN      #1{\unskip}     \fi
\ifx \showLCCN     \undefined \def \showLCCN      #1{\unskip}     \fi
\ifx \shownote     \undefined \def \shownote      #1{#1}          \fi
\ifx \showarticletitle \undefined \def \showarticletitle #1{#1}   \fi
\ifx \showURL      \undefined \def \showURL       {\relax}        \fi
\providecommand\bibfield[2]{#2}
\providecommand\bibinfo[2]{#2}
\providecommand\natexlab[1]{#1}
\providecommand\showeprint[2][]{arXiv:#2}

\bibitem[\protect\citeauthoryear{Abdoos, Mozayani, and Bazzan}{Abdoos
  et~al\mbox{.}}{2011}]%
        {abdoos2011traffic}
\bibfield{author}{\bibinfo{person}{Monireh Abdoos}, \bibinfo{person}{Nasser
  Mozayani}, {and} \bibinfo{person}{Ana~LC Bazzan}.}
  \bibinfo{year}{2011}\natexlab{}.
\newblock \showarticletitle{Traffic light control in non-stationary
  environments based on multi agent Q-learning}. In \bibinfo{booktitle}{{\em
  2011 14th International IEEE conference on intelligent transportation systems
  (ITSC)}}. IEEE, \bibinfo{pages}{1580--1585}.
\newblock


\bibitem[\protect\citeauthoryear{Ahmad, Eckert, and Teredesai}{Ahmad
  et~al\mbox{.}}{2018}]%
        {ahmad2018interpretable}
\bibfield{author}{\bibinfo{person}{Muhammad~Aurangzeb Ahmad},
  \bibinfo{person}{Carly Eckert}, {and} \bibinfo{person}{Ankur Teredesai}.}
  \bibinfo{year}{2018}\natexlab{}.
\newblock \showarticletitle{Interpretable machine learning in healthcare}. In
  \bibinfo{booktitle}{{\em Proceedings of the 2018 ACM International Conference
  on Bioinformatics, Computational Biology, and Health Informatics}}. ACM,
  \bibinfo{pages}{559--560}.
\newblock


\bibitem[\protect\citeauthoryear{Aslani, Seipel, Mesgari, and Wiering}{Aslani
  et~al\mbox{.}}{2018}]%
        {aslani2018traffic}
\bibfield{author}{\bibinfo{person}{Mohammad Aslani}, \bibinfo{person}{Stefan
  Seipel}, \bibinfo{person}{Mohammad~Saadi Mesgari}, {and}
  \bibinfo{person}{Marco Wiering}.} \bibinfo{year}{2018}\natexlab{}.
\newblock \showarticletitle{Traffic signal optimization through discrete and
  continuous reinforcement learning with robustness analysis in downtown
  Tehran}.
\newblock \bibinfo{journal}{{\em Advanced Engineering Informatics\/}}
  \bibinfo{volume}{38} (\bibinfo{year}{2018}), \bibinfo{pages}{639--655}.
\newblock


\bibitem[\protect\citeauthoryear{Bakker, Whiteson, Kester, and Groen}{Bakker
  et~al\mbox{.}}{2010}]%
        {bakker2010traffic}
\bibfield{author}{\bibinfo{person}{Bram Bakker}, \bibinfo{person}{Shimon
  Whiteson}, \bibinfo{person}{Leon Kester}, {and} \bibinfo{person}{Frans~CA
  Groen}.} \bibinfo{year}{2010}\natexlab{}.
\newblock \showarticletitle{Traffic light control by multiagent reinforcement
  learning systems}.
\newblock In \bibinfo{booktitle}{{\em Interactive Collaborative Information
  Systems}}. \bibinfo{publisher}{Springer}, \bibinfo{pages}{475--510}.
\newblock


\bibitem[\protect\citeauthoryear{Bautista, Dy, Ma{\~n}alac, Orbe, and
  Cordel}{Bautista et~al\mbox{.}}{2016}]%
        {bautista2016convolutional}
\bibfield{author}{\bibinfo{person}{Carlo~Migel Bautista},
  \bibinfo{person}{Clifford~Austin Dy}, \bibinfo{person}{Miguel~I{\~n}igo
  Ma{\~n}alac}, \bibinfo{person}{Raphael~Angelo Orbe}, {and}
  \bibinfo{person}{Macario Cordel}.} \bibinfo{year}{2016}\natexlab{}.
\newblock \showarticletitle{Convolutional neural network for vehicle detection
  in low resolution traffic videos}. In \bibinfo{booktitle}{{\em 2016 IEEE
  Region 10 Symposium (TENSYMP)}}. IEEE, \bibinfo{pages}{277--281}.
\newblock


\bibitem[\protect\citeauthoryear{Behrisch, Bieker, Erdmann, and
  Krajzewicz}{Behrisch et~al\mbox{.}}{2011}]%
        {behrisch2011sumo}
\bibfield{author}{\bibinfo{person}{Michael Behrisch}, \bibinfo{person}{Laura
  Bieker}, \bibinfo{person}{Jakob Erdmann}, {and} \bibinfo{person}{Daniel
  Krajzewicz}.} \bibinfo{year}{2011}\natexlab{}.
\newblock \showarticletitle{SUMO--simulation of urban mobility: an overview}.
  In \bibinfo{booktitle}{{\em Proceedings of SIMUL 2011, The Third
  International Conference on Advances in System Simulation}}. ThinkMind.
\newblock


\bibitem[\protect\citeauthoryear{Bonneson, Sunkari, Pratt, Songchitruksa,
  et~al\mbox{.}}{Bonneson et~al\mbox{.}}{2011}]%
        {bonneson2011traffic}
\bibfield{author}{\bibinfo{person}{James Bonneson},
  \bibinfo{person}{Srinivasa~R Sunkari}, \bibinfo{person}{Michael Pratt},
  \bibinfo{person}{Praprut Songchitruksa}, {et~al\mbox{.}}}
  \bibinfo{year}{2011}\natexlab{}.
\newblock \bibinfo{booktitle}{{\em Traffic signal operations handbook}}.
\newblock \bibinfo{type}{{T}echnical {R}eport}. \bibinfo{institution}{Texas.
  Dept. of Transportation. Research and Technology Implementation Office}.
\newblock


\bibitem[\protect\citeauthoryear{Brockman, Cheung, Pettersson, Schneider,
  Schulman, Tang, and Zaremba}{Brockman et~al\mbox{.}}{2016}]%
        {1606.01540}
\bibfield{author}{\bibinfo{person}{Greg Brockman}, \bibinfo{person}{Vicki
  Cheung}, \bibinfo{person}{Ludwig Pettersson}, \bibinfo{person}{Jonas
  Schneider}, \bibinfo{person}{John Schulman}, \bibinfo{person}{Jie Tang},
  {and} \bibinfo{person}{Wojciech Zaremba}.} \bibinfo{year}{2016}\natexlab{}.
\newblock \bibinfo{title}{OpenAI Gym}.
\newblock   (\bibinfo{year}{2016}).
\newblock
\showeprint{arXiv:1606.01540}


\bibitem[\protect\citeauthoryear{Genders and Razavi}{Genders and
  Razavi}{2016}]%
        {genders2016using}
\bibfield{author}{\bibinfo{person}{Wade Genders} {and} \bibinfo{person}{Saiedeh
  Razavi}.} \bibinfo{year}{2016}\natexlab{}.
\newblock \showarticletitle{Using a deep reinforcement learning agent for
  traffic signal control}.
\newblock \bibinfo{journal}{{\em arXiv preprint arXiv:1611.01142\/}}
  (\bibinfo{year}{2016}).
\newblock


\bibitem[\protect\citeauthoryear{Hansen, Akimoto, and Baudis}{Hansen
  et~al\mbox{.}}{2019}]%
        {hansen2019pycma}
\bibfield{author}{\bibinfo{person}{Nikolaus Hansen}, \bibinfo{person}{Youhei
  Akimoto}, {and} \bibinfo{person}{Petr Baudis}.}
  \bibinfo{year}{2019}\natexlab{}.
\newblock \bibinfo{title}{{CMA-ES/pycma} on {G}ithub}.
\newblock \bibinfo{howpublished}{Zenodo, DOI:10.5281/zenodo.2559634}.
  (\bibinfo{date}{Feb.} \bibinfo{year}{2019}).
\newblock
\showDOI{%
\url{https://doi.org/10.5281/zenodo.2559634}}


\bibitem[\protect\citeauthoryear{Hansen and Ostermeier}{Hansen and
  Ostermeier}{2001}]%
        {hansen2001completely}
\bibfield{author}{\bibinfo{person}{Nikolaus Hansen} {and}
  \bibinfo{person}{Andreas Ostermeier}.} \bibinfo{year}{2001}\natexlab{}.
\newblock \showarticletitle{Completely derandomized self-adaptation in
  evolution strategies}.
\newblock \bibinfo{journal}{{\em Evolutionary computation\/}}
  \bibinfo{volume}{9}, \bibinfo{number}{2} (\bibinfo{year}{2001}),
  \bibinfo{pages}{159--195}.
\newblock


\bibitem[\protect\citeauthoryear{Hein, Udluft, and Runkler}{Hein
  et~al\mbox{.}}{2018}]%
        {hein2018interpretable}
\bibfield{author}{\bibinfo{person}{Daniel Hein}, \bibinfo{person}{Steffen
  Udluft}, {and} \bibinfo{person}{Thomas~A Runkler}.}
  \bibinfo{year}{2018}\natexlab{}.
\newblock \showarticletitle{Interpretable policies for reinforcement learning
  by genetic programming}.
\newblock \bibinfo{journal}{{\em Engineering Applications of Artificial
  Intelligence\/}}  \bibinfo{volume}{76} (\bibinfo{year}{2018}),
  \bibinfo{pages}{158--169}.
\newblock


\bibitem[\protect\citeauthoryear{Iwasaki, Kawata, and Nakamiya}{Iwasaki
  et~al\mbox{.}}{2011}]%
        {iwasaki2011robust}
\bibfield{author}{\bibinfo{person}{Yoichiro Iwasaki}, \bibinfo{person}{Shinya
  Kawata}, {and} \bibinfo{person}{Toshiyuki Nakamiya}.}
  \bibinfo{year}{2011}\natexlab{}.
\newblock \showarticletitle{Robust vehicle detection even in poor visibility
  conditions using infrared thermal images and its application to road traffic
  flow monitoring}.
\newblock \bibinfo{journal}{{\em Measurement Science and Technology\/}}
  \bibinfo{volume}{22}, \bibinfo{number}{8} (\bibinfo{year}{2011}),
  \bibinfo{pages}{085501}.
\newblock


\bibitem[\protect\citeauthoryear{Konidaris, Osentoski, and Thomas}{Konidaris
  et~al\mbox{.}}{2011}]%
        {konidaris2011value}
\bibfield{author}{\bibinfo{person}{George Konidaris}, \bibinfo{person}{Sarah
  Osentoski}, {and} \bibinfo{person}{Philip Thomas}.}
  \bibinfo{year}{2011}\natexlab{}.
\newblock \showarticletitle{Value function approximation in reinforcement
  learning using the Fourier basis}. In \bibinfo{booktitle}{{\em Twenty-fifth
  AAAI conference on artificial intelligence}}.
\newblock


\bibitem[\protect\citeauthoryear{Kostakos, Ojala, and Juntunen}{Kostakos
  et~al\mbox{.}}{2013}]%
        {kostakos2013traffic}
\bibfield{author}{\bibinfo{person}{Vassilis Kostakos}, \bibinfo{person}{Timo
  Ojala}, {and} \bibinfo{person}{Tomi Juntunen}.}
  \bibinfo{year}{2013}\natexlab{}.
\newblock \showarticletitle{Traffic in the smart city: Exploring city-wide
  sensing for traffic control center augmentation}.
\newblock \bibinfo{journal}{{\em IEEE Internet Computing\/}}
  \bibinfo{volume}{17}, \bibinfo{number}{6} (\bibinfo{year}{2013}),
  \bibinfo{pages}{22--29}.
\newblock


\bibitem[\protect\citeauthoryear{Lagoudakis and Parr}{Lagoudakis and
  Parr}{2003}]%
        {lagoudakis2003least}
\bibfield{author}{\bibinfo{person}{Michail~G Lagoudakis} {and}
  \bibinfo{person}{Ronald Parr}.} \bibinfo{year}{2003}\natexlab{}.
\newblock \showarticletitle{Least-squares policy iteration}.
\newblock \bibinfo{journal}{{\em Journal of machine learning research\/}}
  \bibinfo{volume}{4}, \bibinfo{number}{Dec} (\bibinfo{year}{2003}),
  \bibinfo{pages}{1107--1149}.
\newblock


\bibitem[\protect\citeauthoryear{Laval and Zhou}{Laval and Zhou}{2019}]%
        {laval2019large}
\bibfield{author}{\bibinfo{person}{Jorge~A Laval} {and} \bibinfo{person}{Hao
  Zhou}.} \bibinfo{year}{2019}\natexlab{}.
\newblock \showarticletitle{Large-scale traffic signal control using machine
  learning: some traffic flow considerations}.
\newblock \bibinfo{journal}{{\em arXiv preprint arXiv:1908.02673\/}}
  (\bibinfo{year}{2019}).
\newblock


\bibitem[\protect\citeauthoryear{Levinson}{Levinson}{1998}]%
        {levinson1998speed}
\bibfield{author}{\bibinfo{person}{David~M Levinson}.}
  \bibinfo{year}{1998}\natexlab{}.
\newblock \showarticletitle{Speed and delay on signalized arterials}.
\newblock \bibinfo{journal}{{\em Journal of Transportation Engineering\/}}
  \bibinfo{volume}{124}, \bibinfo{number}{3} (\bibinfo{year}{1998}),
  \bibinfo{pages}{258--263}.
\newblock


\bibitem[\protect\citeauthoryear{Li, Lv, and Wang}{Li et~al\mbox{.}}{2016}]%
        {li2016traffic}
\bibfield{author}{\bibinfo{person}{Li Li}, \bibinfo{person}{Yisheng Lv}, {and}
  \bibinfo{person}{Fei-Yue Wang}.} \bibinfo{year}{2016}\natexlab{}.
\newblock \showarticletitle{Traffic signal timing via deep reinforcement
  learning}.
\newblock \bibinfo{journal}{{\em IEEE/CAA Journal of Automatica Sinica\/}}
  \bibinfo{volume}{3}, \bibinfo{number}{3} (\bibinfo{year}{2016}),
  \bibinfo{pages}{247--254}.
\newblock


\bibitem[\protect\citeauthoryear{Liang, Du, Wang, and Han}{Liang
  et~al\mbox{.}}{2018}]%
        {liang2018deep}
\bibfield{author}{\bibinfo{person}{Xiaoyuan Liang}, \bibinfo{person}{Xunsheng
  Du}, \bibinfo{person}{Guiling Wang}, {and} \bibinfo{person}{Zhu Han}.}
  \bibinfo{year}{2018}\natexlab{}.
\newblock \showarticletitle{Deep reinforcement learning for traffic light
  control in vehicular networks}.
\newblock \bibinfo{journal}{{\em arXiv preprint arXiv:1803.11115\/}}
  (\bibinfo{year}{2018}).
\newblock


\bibitem[\protect\citeauthoryear{Liu, Schulte, Zhu, and Li}{Liu
  et~al\mbox{.}}{2018}]%
        {liu2018toward}
\bibfield{author}{\bibinfo{person}{Guiliang Liu}, \bibinfo{person}{Oliver
  Schulte}, \bibinfo{person}{Wang Zhu}, {and} \bibinfo{person}{Qingcan Li}.}
  \bibinfo{year}{2018}\natexlab{}.
\newblock \showarticletitle{Toward interpretable deep reinforcement learning
  with linear model u-trees}. In \bibinfo{booktitle}{{\em Joint European
  Conference on Machine Learning and Knowledge Discovery in Databases}}.
  Springer, \bibinfo{pages}{414--429}.
\newblock


\bibitem[\protect\citeauthoryear{Meyer, Hinz, Laika, Weihing, and Bamler}{Meyer
  et~al\mbox{.}}{2006}]%
        {meyer2006performance}
\bibfield{author}{\bibinfo{person}{Franz Meyer}, \bibinfo{person}{Stefan Hinz},
  \bibinfo{person}{Andreas Laika}, \bibinfo{person}{Diana Weihing}, {and}
  \bibinfo{person}{Richard Bamler}.} \bibinfo{year}{2006}\natexlab{}.
\newblock \showarticletitle{Performance analysis of the TerraSAR-X traffic
  monitoring concept}.
\newblock \bibinfo{journal}{{\em ISPRS Journal of Photogrammetry and Remote
  Sensing\/}} \bibinfo{volume}{61}, \bibinfo{number}{3-4}
  (\bibinfo{year}{2006}), \bibinfo{pages}{225--242}.
\newblock


\bibitem[\protect\citeauthoryear{Mnih, Kavukcuoglu, Silver, Rusu, Veness,
  Bellemare, Graves, Riedmiller, Fidjeland, Ostrovski, et~al\mbox{.}}{Mnih
  et~al\mbox{.}}{2015}]%
        {mnih2015human}
\bibfield{author}{\bibinfo{person}{Volodymyr Mnih}, \bibinfo{person}{Koray
  Kavukcuoglu}, \bibinfo{person}{David Silver}, \bibinfo{person}{Andrei~A
  Rusu}, \bibinfo{person}{Joel Veness}, \bibinfo{person}{Marc~G Bellemare},
  \bibinfo{person}{Alex Graves}, \bibinfo{person}{Martin Riedmiller},
  \bibinfo{person}{Andreas~K Fidjeland}, \bibinfo{person}{Georg Ostrovski},
  {et~al\mbox{.}}} \bibinfo{year}{2015}\natexlab{}.
\newblock \showarticletitle{Human-level control through deep reinforcement
  learning}.
\newblock \bibinfo{journal}{{\em Nature\/}} \bibinfo{volume}{518},
  \bibinfo{number}{7540} (\bibinfo{year}{2015}), \bibinfo{pages}{529}.
\newblock


\bibitem[\protect\citeauthoryear{Mohan, Padmanabhan, and Ramjee}{Mohan
  et~al\mbox{.}}{2008}]%
        {mohan2008nericell}
\bibfield{author}{\bibinfo{person}{Prashanth Mohan}, \bibinfo{person}{Venkata~N
  Padmanabhan}, {and} \bibinfo{person}{Ramachandran Ramjee}.}
  \bibinfo{year}{2008}\natexlab{}.
\newblock \showarticletitle{Nericell: rich monitoring of road and traffic
  conditions using mobile smartphones}. In \bibinfo{booktitle}{{\em Proceedings
  of the 6th ACM conference on Embedded network sensor systems}}. ACM,
  \bibinfo{pages}{323--336}.
\newblock


\bibitem[\protect\citeauthoryear{Mousavi, Schukat, and Howley}{Mousavi
  et~al\mbox{.}}{2017}]%
        {mousavi2017traffic}
\bibfield{author}{\bibinfo{person}{Seyed~Sajad Mousavi},
  \bibinfo{person}{Michael Schukat}, {and} \bibinfo{person}{Enda Howley}.}
  \bibinfo{year}{2017}\natexlab{}.
\newblock \showarticletitle{Traffic light control using deep policy-gradient
  and value-function-based reinforcement learning}.
\newblock \bibinfo{journal}{{\em IET Intelligent Transport Systems\/}}
  \bibinfo{volume}{11}, \bibinfo{number}{7} (\bibinfo{year}{2017}),
  \bibinfo{pages}{417--423}.
\newblock


\bibitem[\protect\citeauthoryear{Prabuchandran, Hemanth, and
  Bhatnagar}{Prabuchandran et~al\mbox{.}}{2014}]%
        {prabuchandran2014multi}
\bibfield{author}{\bibinfo{person}{K.~J. Prabuchandran}, \bibinfo{person}{Kumar
  A.~N. Hemanth}, {and} \bibinfo{person}{Shalabh Bhatnagar}.}
  \bibinfo{year}{2014}\natexlab{}.
\newblock \showarticletitle{Multi-agent reinforcement learning for traffic
  signal control}. In \bibinfo{booktitle}{{\em 17th International IEEE
  Conference on Intelligent Transportation Systems (ITSC)}}. IEEE,
  \bibinfo{pages}{2529--2534}.
\newblock


\bibitem[\protect\citeauthoryear{Prashanth and Bhatnagar}{Prashanth and
  Bhatnagar}{2011}]%
        {prashanth2011reinforcement}
\bibfield{author}{\bibinfo{person}{LA Prashanth} {and} \bibinfo{person}{Shalabh
  Bhatnagar}.} \bibinfo{year}{2011}\natexlab{}.
\newblock \showarticletitle{Reinforcement learning with function approximation
  for traffic signal control}.
\newblock \bibinfo{journal}{{\em IEEE Transactions on Intelligent
  Transportation Systems\/}} \bibinfo{volume}{12}, \bibinfo{number}{2}
  (\bibinfo{year}{2011}), \bibinfo{pages}{412--421}.
\newblock


\bibitem[\protect\citeauthoryear{Samczynski, Kulpa, Malanowski, Krysik,
  et~al\mbox{.}}{Samczynski et~al\mbox{.}}{2011}]%
        {samczynski2011concept}
\bibfield{author}{\bibinfo{person}{P Samczynski}, \bibinfo{person}{K Kulpa},
  \bibinfo{person}{M Malanowski}, \bibinfo{person}{P Krysik}, {et~al\mbox{.}}}
  \bibinfo{year}{2011}\natexlab{}.
\newblock \showarticletitle{A concept of GSM-based passive radar for vehicle
  traffic monitoring}. In \bibinfo{booktitle}{{\em 2011 MICROWAVES, RADAR AND
  REMOTE SENSING SYMPOSIUM}}. IEEE, \bibinfo{pages}{271--274}.
\newblock


\bibitem[\protect\citeauthoryear{Samek, Wiegand, and M{\"u}ller}{Samek
  et~al\mbox{.}}{2017}]%
        {samek2017explainable}
\bibfield{author}{\bibinfo{person}{Wojciech Samek}, \bibinfo{person}{Thomas
  Wiegand}, {and} \bibinfo{person}{Klaus-Robert M{\"u}ller}.}
  \bibinfo{year}{2017}\natexlab{}.
\newblock \showarticletitle{Explainable artificial intelligence: Understanding,
  visualizing and interpreting deep learning models}.
\newblock \bibinfo{journal}{{\em arXiv preprint arXiv:1708.08296\/}}
  (\bibinfo{year}{2017}).
\newblock


\bibitem[\protect\citeauthoryear{Schulman, Levine, Abbeel, Jordan, and
  Moritz}{Schulman et~al\mbox{.}}{2015}]%
        {schulman2015trust}
\bibfield{author}{\bibinfo{person}{John Schulman}, \bibinfo{person}{Sergey
  Levine}, \bibinfo{person}{Pieter Abbeel}, \bibinfo{person}{Michael Jordan},
  {and} \bibinfo{person}{Philipp Moritz}.} \bibinfo{year}{2015}\natexlab{}.
\newblock \showarticletitle{Trust region policy optimization}. In
  \bibinfo{booktitle}{{\em International Conference on Machine Learning}}.
  \bibinfo{pages}{1889--1897}.
\newblock


\bibitem[\protect\citeauthoryear{Schulman, Wolski, Dhariwal, Radford, and
  Klimov}{Schulman et~al\mbox{.}}{2017}]%
        {schulman2017proximal}
\bibfield{author}{\bibinfo{person}{John Schulman}, \bibinfo{person}{Filip
  Wolski}, \bibinfo{person}{Prafulla Dhariwal}, \bibinfo{person}{Alec Radford},
  {and} \bibinfo{person}{Oleg Klimov}.} \bibinfo{year}{2017}\natexlab{}.
\newblock \showarticletitle{Proximal policy optimization algorithms}.
\newblock \bibinfo{journal}{{\em arXiv preprint arXiv:1707.06347\/}}
  (\bibinfo{year}{2017}).
\newblock


\bibitem[\protect\citeauthoryear{Shabestary and Abdulhai}{Shabestary and
  Abdulhai}{2018}]%
        {shabestary2018deep}
\bibfield{author}{\bibinfo{person}{Soheil Mohamad~Alizadeh Shabestary} {and}
  \bibinfo{person}{Baher Abdulhai}.} \bibinfo{year}{2018}\natexlab{}.
\newblock \showarticletitle{Deep Learning vs. Discrete Reinforcement Learning
  for Adaptive Traffic Signal Control}. In \bibinfo{booktitle}{{\em 2018 21st
  International Conference on Intelligent Transportation Systems (ITSC)}}.
  IEEE, \bibinfo{pages}{286--293}.
\newblock


\bibitem[\protect\citeauthoryear{Singh and Sutton}{Singh and Sutton}{1996}]%
        {singh1996reinforcement}
\bibfield{author}{\bibinfo{person}{Satinder~P Singh} {and}
  \bibinfo{person}{Richard~S Sutton}.} \bibinfo{year}{1996}\natexlab{}.
\newblock \showarticletitle{Reinforcement learning with replacing eligibility
  traces}.
\newblock \bibinfo{journal}{{\em Machine learning\/}} \bibinfo{volume}{22},
  \bibinfo{number}{1-3} (\bibinfo{year}{1996}), \bibinfo{pages}{123--158}.
\newblock


\bibitem[\protect\citeauthoryear{Sutton and Barto}{Sutton and Barto}{2018}]%
        {sutton2018reinforcement}
\bibfield{author}{\bibinfo{person}{Richard~S Sutton} {and}
  \bibinfo{person}{Andrew~G Barto}.} \bibinfo{year}{2018}\natexlab{}.
\newblock \bibinfo{booktitle}{{\em Reinforcement learning: An introduction}}.
\newblock \bibinfo{publisher}{MIT press}.
\newblock


\bibitem[\protect\citeauthoryear{Tirachini}{Tirachini}{2013}]%
        {tirachini2013estimation}
\bibfield{author}{\bibinfo{person}{Alejandro Tirachini}.}
  \bibinfo{year}{2013}\natexlab{}.
\newblock \showarticletitle{Estimation of travel time and the benefits of
  upgrading the fare payment technology in urban bus services}.
\newblock \bibinfo{journal}{{\em Transportation Research Part C: Emerging
  Technologies\/}}  \bibinfo{volume}{30} (\bibinfo{year}{2013}),
  \bibinfo{pages}{239--256}.
\newblock


\bibitem[\protect\citeauthoryear{Van~der Pol and Oliehoek}{Van~der Pol and
  Oliehoek}{2016}]%
        {van2016coordinated}
\bibfield{author}{\bibinfo{person}{Elise Van~der Pol} {and}
  \bibinfo{person}{Frans~A Oliehoek}.} \bibinfo{year}{2016}\natexlab{}.
\newblock \showarticletitle{Coordinated deep reinforcement learners for traffic
  light control}.
\newblock \bibinfo{journal}{{\em Proceedings of Learning, Inference and Control
  of Multi-Agent Systems (at NIPS 2016)\/}} (\bibinfo{year}{2016}).
\newblock


\bibitem[\protect\citeauthoryear{Verma, Murali, Singh, Kohli, and
  Chaudhuri}{Verma et~al\mbox{.}}{2018}]%
        {verma2018programmatically}
\bibfield{author}{\bibinfo{person}{Abhinav Verma},
  \bibinfo{person}{Vijayaraghavan Murali}, \bibinfo{person}{Rishabh Singh},
  \bibinfo{person}{Pushmeet Kohli}, {and} \bibinfo{person}{Swarat Chaudhuri}.}
  \bibinfo{year}{2018}\natexlab{}.
\newblock \showarticletitle{Programmatically interpretable reinforcement
  learning}.
\newblock \bibinfo{journal}{{\em arXiv preprint arXiv:1804.02477\/}}
  (\bibinfo{year}{2018}).
\newblock


\bibitem[\protect\citeauthoryear{Wan, Liu, Shao, Vasilakos, Imran, and
  Zhou}{Wan et~al\mbox{.}}{2016}]%
        {wan2016mobile}
\bibfield{author}{\bibinfo{person}{Jiafu Wan}, \bibinfo{person}{Jianqi Liu},
  \bibinfo{person}{Zehui Shao}, \bibinfo{person}{Athanasios Vasilakos},
  \bibinfo{person}{Muhammad Imran}, {and} \bibinfo{person}{Keliang Zhou}.}
  \bibinfo{year}{2016}\natexlab{}.
\newblock \showarticletitle{Mobile crowd sensing for traffic prediction in
  internet of vehicles}.
\newblock \bibinfo{journal}{{\em Sensors\/}} \bibinfo{volume}{16},
  \bibinfo{number}{1} (\bibinfo{year}{2016}), \bibinfo{pages}{88}.
\newblock


\bibitem[\protect\citeauthoryear{Wang, Ke, Qiao, and Chai}{Wang
  et~al\mbox{.}}{2019}]%
        {wang2019large}
\bibfield{author}{\bibinfo{person}{Xiaoqiang Wang}, \bibinfo{person}{Liangjun
  Ke}, \bibinfo{person}{Zhimin Qiao}, {and} \bibinfo{person}{Xinghua Chai}.}
  \bibinfo{year}{2019}\natexlab{}.
\newblock \showarticletitle{Large-scale Traffic Signal Control Using a Novel
  Multi-Agent Reinforcement Learning}.
\newblock \bibinfo{journal}{{\em arXiv preprint arXiv:1908.03761\/}}
  (\bibinfo{year}{2019}).
\newblock


\bibitem[\protect\citeauthoryear{Wei, Zheng, Yao, and Li}{Wei
  et~al\mbox{.}}{2018}]%
        {wei2018intellilight}
\bibfield{author}{\bibinfo{person}{Hua Wei}, \bibinfo{person}{Guanjie Zheng},
  \bibinfo{person}{Huaxiu Yao}, {and} \bibinfo{person}{Zhenhui Li}.}
  \bibinfo{year}{2018}\natexlab{}.
\newblock \showarticletitle{Intellilight: A reinforcement learning approach for
  intelligent traffic light control}. In \bibinfo{booktitle}{{\em Proceedings
  of the 24th ACM SIGKDD International Conference on Knowledge Discovery \&
  Data Mining}}. ACM, \bibinfo{pages}{2496--2505}.
\newblock


\bibitem[\protect\citeauthoryear{Williams}{Williams}{1992}]%
        {williams1992simple}
\bibfield{author}{\bibinfo{person}{Ronald~J Williams}.}
  \bibinfo{year}{1992}\natexlab{}.
\newblock \showarticletitle{Simple statistical gradient-following algorithms
  for connectionist reinforcement learning}.
\newblock \bibinfo{journal}{{\em Machine learning\/}} \bibinfo{volume}{8},
  \bibinfo{number}{3-4} (\bibinfo{year}{1992}), \bibinfo{pages}{229--256}.
\newblock


\end{thebibliography}

\end{document}